%% file: main.tex
\documentclass[letterpaper,twocolumn,10pt]{article}
\usepackage{usenix2019_v3}

\usepackage{lmodern}
\usepackage{booktabs}
\usepackage{array}
\usepackage{microtype}
\usepackage{amsmath,amssymb,amsfonts}
\usepackage{amsthm}
\usepackage{latexsym}
\usepackage{listings}
\usepackage{float}
\usepackage{xspace}
\usepackage{subcaption}

\AtBeginDocument{\DeclareMathAlphabet{\mathcal}{OMS}{cmsy}{m}{n}}

\theoremstyle{definition}
\newtheorem{definition}{Definition}
\newtheorem{theorem}{Theorem}

\newenvironment{artifactblock}
{\par\noindent\small\begin{tabular}{@{}>{\raggedright\arraybackslash}p{0.98\columnwidth}@{}}}
{\end{tabular}\par}

\newcommand{\PaperTitle}{Replicating Belief, Not Bits: Epistemic State Replication for Agentic Systems}

\hypersetup{
  pdftitle={\PaperTitle},
  pdfauthor={Jun He and Deying Yu}
}

\title{\bf \PaperTitle}

\author{
  {\rm Jun He}\\
  OpenKedge.io
  \and
  {\rm Deying Yu}\\
  OpenKedge.io
}

\begin{document}
\markboth{\PaperTitle}{\PaperTitle}

\maketitle

\begin{abstract}
In distributed systems, the classical State Machine Replication (SMR) model assumes that correct replicas execute deterministic transitions to yield identical bitwise states. However, the rise of agentic distributed systems---where autonomous, stochastic, and model-driven agents orchestrate infrastructure---presents scenarios where deterministic, bitwise replication is insufficient. Replicas operating with generative models may exhibit divergent reasoning paths, summaries, and token boundaries, yet reach semantically equivalent and correct operational decisions. Forcing bitwise agreement across these stochastic participants degrades execution flexibility, induces context amnesia, and limits performance.

We argue that in such settings replicas should agree on \emph{belief}, not \emph{bits}. We propose \textbf{Epistemic State Replication (ESR)}, a belief-replication layer for agentic distributed systems that shifts the replication boundary from data visibility to knowledge visibility. We formalize the epistemic node state as a pair $\mathcal{K} = (\mathcal{L}, \mathcal{B})$ separating the deterministic, immutable evidence log ($\mathcal{L}$) from the stochastic, evolving belief lineage ($\mathcal{B}$). To govern execution safety, we define \emph{Semantic Linearizability}, which requires operations to reflect the latest committed operational meaning within a verifier-bounded semantic compatibility metric, and \emph{Bounded Eventual Coherence}, which bounds expected semantic divergence under fair delivery, monotonic evidence, bounded verifier disturbance, and a contractive graft operator. We outline protocols for propagating derived insights using structured \emph{epistemic deltas}, and formalize \emph{Verifiable Semantic Rollbacks} to prune faulty premises from belief lineages without inducing context amnesia. We prototype ESR and report preliminary simulation results that show feasibility under the stated assumptions and illustrate reductions in secondary cognitive faults.
\end{abstract}

\input{sections/01-introduction}
\input{sections/02-node-model}
\input{sections/03-consistency-linearizability}
\input{sections/04-consistency-protocols}
\input{sections/05-semantic-rollbacks}
\input{sections/06-implementation-evaluation}
\input{sections/07-related-work}

\bibliographystyle{unsrt}
\bibliography{refs}

\end{document}

%% file: sections/01-introduction.tex
\section{Introduction}\label{sec:intro}

Distributed systems have historically operated under a deterministic consensus model. In the classical State Machine Replication (SMR) paradigm~\cite{schneider1990replication}, a set of replicas achieves reliability by starting from identical initial states and executing the same sequence of deterministic state transitions. This model relies on a fundamental axiom: correct participants execute identical code to produce identical bitwise states. Formally, for two correct replicas $A$ and $B$, their states at any logical step $t$ must satisfy a strict byte-level equality:
\begin{equation}
S_A(t) = S_B(t).
\end{equation}
This deterministic axiom has formed the basis of classical distributed systems, from consensus protocols like Paxos and Raft~\cite{ongaro2014raft} to distributed databases and transaction managers~\cite{bernstein1987concurrency}.

However, the integration of large language models (LLMs) and autonomous reasoning agents into cloud control planes, software delivery pipelines, and critical infrastructure presents scenarios where deterministic, bitwise replication is insufficient~\cite{pdds2026manifesto}. Replicas operating with generative models may exhibit divergent reasoning paths, summaries, and token boundaries, yet reach semantically equivalent and correct operational decisions. Forcing bitwise state agreement ($S_A = S_B$) across these stochastic replicas is computationally impractical and semantically restrictive. Generative models operate over large token spaces where minor variations in phrasing, summary structures, or token boundaries do not alter the underlying operational intent. If a replication layer rejects transitions that are not bitwise identical, it destroys the cognitive diversity and reasoning flexibility of the participant group, triggering execution stalls. Conversely, if the system allows replicas to drift without any consistency bounds, it risks severe coordination failures.

At the same time, existing agentic frameworks manage memory through local recall mechanisms (e.g., conversation histories, sliding context windows, and vector databases). While these systems support individual agent retrieval, they lack distributed consistency, semantic coordination across replicas, and formal rollback recovery. When an agent experiences a failure, reverting database rows (a physical rollback) is insufficient because the agent's internal memory remains contaminated with flawed belief contexts. Conversely, naive pruning of agent context to fit window limits leads to \emph{context amnesia}: bounded context management removes or summarizes away evidence or beliefs that remain causally relevant to future safe decisions. ESR addresses this failure mode by recording explicit justification links between beliefs and the evidence or prior beliefs on which they depend.

To address these limitations, we propose \textbf{Epistemic State Replication (ESR)}, a belief-replication layer for agentic distributed systems. The core thesis is that ESR replicates \emph{belief}, not \emph{bits}: replicas need not agree on token sequences or prompt buffers, only on what those tokens \emph{mean} operationally. Concretely, ESR shifts the replication boundary from data visibility to \emph{knowledge visibility} by decoupling physical state representations (such as token sequences or prompt buffers) from semantic knowledge retention. Data visibility exposes committed physical observations; state visibility exposes concrete byte-level or object-level state; knowledge visibility exposes semantically compatible operational beliefs, evidence references, and admissible action boundaries. Under knowledge visibility, replicas need not expose identical prompts, token buffers, or intermediate reasoning traces. Rather than requiring bitwise equality, ESR maps physical states to their operational meaning or epistemic intent via an evaluation projection $E(S)$. The correctness condition for replication is expressed as an approximate semantic relation:
\begin{equation}
d_{\mathcal{M}}(E(S_A), E(S_B)) \le \epsilon,
\end{equation}
where $d_{\mathcal{M}}$ is a semantic distance metric. The intent is to keep replicas within an admissible operational compatibility class even when their physical token sequences diverge.

This paper formalizes the system model, consistency protocols, and recovery operators of ESR. We model the epistemic node state by separating raw telemetry from belief lineage, define semantic consistency models, and propose mechanisms for propagating insights and performing verifiable rollbacks.

\noindent\textbf{Contributions.} In this paper, we make the following contributions:
\begin{itemize}
    \item \textbf{The Epistemic Node and System Model:} We formalize the anatomy of reasoning replicas, defining the \emph{Immutable Evidence Log}, the \emph{Belief Lineage}, and a comprehensive failure model for stochastic distributed execution.
    \item \textbf{Consistency Specifications:} We formalize \emph{Semantic Linearizability} and \emph{Bounded Eventual Coherence} in post-deterministic environments, framing convergence in terms of bounded semantic compatibility.
    \item \textbf{Epistemic Consistency Protocols:} We outline methods to propagate derived reasoning across agent clusters using lightweight \emph{epistemic deltas} and a context grafting operator ($\oplus$), providing pseudocode for core operations.
    \item \textbf{Verifiable Semantic Rollbacks:} We formalize a rollback protocol that alters an agent's belief lineage to prune failed premises without causing context amnesia, and we state conditional safety invariants over the repaired lineage.
    \item \textbf{System Evaluation in Simulation:} We prototype the persistence layer and report preliminary simulation results in a cloud-control-plane setting, covering throughput, latency, and secondary-fault behavior.
\end{itemize}

%% file: sections/02-node-model.tex
\section{The Epistemic Node Model}

To model reasoning replicas in a distributed system, we must redefine what constitutes the state of a node. In classical state-machine replication, a node's state is a single bitwise representation. In an agentic distributed system, however, the node's state includes both physical observations of its environment and its internal cognitive interpretation of those observations. We formalize this dual state as the \emph{Epistemic Node Model}.

\subsection{Anatomy of a Reasoning Replica}

An epistemic node $i$ maintains a knowledge state $\mathcal{K}_i(t)$ at any logical time $t$. We define this state as a tuple:
\begin{equation}
\mathcal{K}_i(t) = (\mathcal{L}_i(t), \mathcal{B}_i(t)),
\end{equation}
where:
\begin{itemize}
    \item $\mathcal{L}_i(t)$ is the \textbf{Immutable Evidence Log}. It contains the hard telemetry, physical transaction logs, and external events observed by the replica. It is deterministic, append-only, and identical across correct replicas.
    \item $\mathcal{B}_i(t)$ is the \textbf{Belief Lineage}. It represents the node's stochastic interpretation of the evidence. It is composed of summaries, retrieved documents, plans, and inferences. Unlike the log, the belief lineage can diverge across different replicas while remaining semantically equivalent.
\end{itemize}

\subsection{Formalizing the Belief Lineage}

The Belief Lineage $\mathcal{B}_i(t)$ is formally structured as a directed acyclic graph (DAG) of beliefs:
\begin{equation}
\mathcal{B}_i(t) = (V_B, E_B),
\end{equation}
where each vertex $b \in V_B$ represents a distinct belief unit, defined as a triplet:
\begin{equation}
b = (\phi, \mathcal{J}, \tau),
\end{equation}
where $\phi$ is the semantic content of the belief (e.g., a natural language proposition or a structured key-value state), $\mathcal{J} \subseteq \mathcal{L}_i(t) \cup V_B$ is the \emph{justification set} containing the causal ancestors (evidence or prior beliefs) that led to this belief, and $\tau$ is the logical timestamp of its derivation.

The directed edges $e = (b_m, b_n) \in E_B$ represent logical derivation steps, meaning belief $b_n$ was derived in part using belief $b_m$ (i.e., $b_m \in \mathcal{J}_n$). This causal linking is what allows ESR to avoid context amnesia: it explicitly tracks which raw inputs or intermediate assumptions led to a given operational decision.

\subsection{The Cognitive Transaction}
\label{sec:cognitive-transaction}

Mutations to the knowledge state $\mathcal{K}_i$ are performed via a \textbf{Cognitive Transaction} ($CT$). Unlike database write transactions, which apply deterministic state transitions, a cognitive transaction encapsulates a stochastic reasoning step executed by an LLM or an agent.

Formally, a Cognitive Transaction is a mapping:
\begin{equation}
CT: \mathcal{I} \times (\mathcal{L} \times \mathcal{B}) \xrightarrow{\theta} \mathcal{A} \times \mathcal{R} \times \Delta\mathcal{B},
\end{equation}
where $\mathcal{I}$ is the input intent (e.g., a user query or system alert), $\mathcal{L} \times \mathcal{B}$ is the active knowledge state, $\theta$ represents the stochastic parameters of the reasoning engine (such as model temperature and seed), $\mathcal{A}$ is the proposed external action or system mutation, $\mathcal{R}$ is the generated reasoning trace (raw tokens), and $\Delta\mathcal{B}$ is the set of new beliefs to be committed to the lineage.

The lifecycle of a cognitive transaction $CT$ is modeled as a state machine:
\begin{enumerate}
    \item \textbf{Proposed:} The replica assembles evidence from $\mathcal{L}_i$ and active beliefs from $\mathcal{B}_i$ into a prompt context $C$, then runs the reasoning engine under intent $\mathcal{I}$ and parameter $\theta$ to produce a proposed action $\mathcal{A}$, trace $\mathcal{R}$, and belief delta $\Delta\mathcal{B}$.
    \item \textbf{Pre-validated:} A semantic quorum checks the proposed external action against the safety policy boundary $\mathcal{P}$. This is a pre-execution safety gate: it authorizes, rejects, or escalates $\mathcal{A}$ before the environment is mutated.
    \item \textbf{Executed:} If pre-validation succeeds, the system executes $\mathcal{A}$ against the external environment and records the observed physical outcome.
    \item \textbf{Evidence-committed:} The physical observation and any resulting system mutation are appended to the shared evidence log $\mathcal{L}$ using the underlying consensus protocol.
    \item \textbf{Belief-committed:} The system extracts the subset of $C$ causally relevant to $\mathcal{A}$, records it as the justification set $\mathcal{J}$, and inserts $\Delta\mathcal{B}$ into the lineage as a pending or active belief update according to the commit policy.
    \item \textbf{Post-validated:} A semantic quorum checks that the resulting belief projection remains compatible with peer projections and policy constraints. This is a post-commit coherence check, distinct from the pre-execution safety check.
    \item \textbf{Finalized or Rolled back:} If post-validation succeeds, the transaction is finalized. If post-validation fails or execution faults, the transaction enters the rollback protocol, which repairs or removes affected beliefs before the autonomous loop resumes.
\end{enumerate}

The \emph{semantic linearization point} of a cognitive transaction is the transition to \textbf{Finalized}: the physical outcome has an evidence-log position, the belief delta has been committed to the lineage, and post-commit semantic validation has certified the resulting projection. The pre-validation state is therefore not a semantic linearization point; it is only an execution-safety gate for the proposed action.

By structuring node state and transactions in this manner, ESR makes the justification path for any finalized action traceable even though the reasoning trace $\mathcal{R}$ is stochastic and heterogeneous.

\subsection{System and Failure Model}

We assume a distributed execution environment consisting of a set of $N$ replicated agent controllers, a shared and replicated evidence log, stochastic reasoning engines, and semantic validators.

\subsubsection{Environment Model}
Replicas interact with each other and the external environment through three abstractions:
\begin{itemize}
    \item \textbf{Replicated Agent Controllers:} A set of processes $\{p_1, p_2, \dots, p_N\}$ that coordinate state mutations. Replicas run stochastic reasoning engines (e.g., LLMs parameterized by temperature $\theta > 0$) to decide on mutations.
    \item \textbf{Shared Evidence Log ($\mathcal{L}$):} An append-only log that serves as the single source of truth for physical observations. The log is replicated across nodes using a standard consensus mechanism (e.g., Raft).
    \item \textbf{Semantic Validators:} Code modules that evaluate reasoning output or state projections against a set of safety policies $\mathcal{P}$.
    \item \textbf{External Actions ($\mathcal{A}$):} Operations executed on the external physical environment (e.g., infrastructure orchestration, network routing updates).
\end{itemize}

\subsubsection{Fault Classifications}
Unlike classical distributed systems, which focus primarily on network partitions and crash-stop faults, an agentic distributed system must handle both physical and cognitive failure modes. We explicitly distinguish the following categories:
\begin{enumerate}
    \item \textbf{Crash-Stop Faults ($f_{crash}$):} A replica stops executing operations.
    \item \textbf{Stale Evidence ($f_{stale}$):} A replica reads a prefix of $\mathcal{L}$ that lags behind the global consensus state, leading to outdated reasoning context.
    \item \textbf{Hallucinated Premises ($f_{halluc}$):} The stochastic reasoning engine generates a proposition $\phi$ that has no supporting evidence in the log $\mathcal{L}$.
    \item \textbf{Poisoned Beliefs ($f_{poison}$):} An incorrect or unvalidated belief is committed to a replica's lineage and propagated to other replicas.
    \item \textbf{Model Divergence ($f_{div}$):} Nodes run heterogeneous models or experience stochastic drift, causing their reasoning traces to diverge, even under identical contexts.
    \item \textbf{Semantic Verifier Errors ($f_{verifier}$):} A validator incorrectly approves an unsafe action (false negative) or rejects a safe, semantically valid action (false positive).
    \item \textbf{Rollback-Triggering Faults ($f_{rollback}$):} Physical execution failures in the environment (e.g., container crash during deployment) that require reverting both physical state and the memory premises that authorized the action.
\end{enumerate}

\subsubsection{Trust Boundary}
To state the safety and liveness assumptions precisely, we define the trusted system components. The \emph{Trusted Computing Base} (TCB) consists of:
\begin{itemize}
    \item The physical consensus layer replicating the evidence log $\mathcal{L}$, ensuring its immutability and linearizable order.
    \item The quorum mechanism that certifies semantic replication commits.
    \item The policy boundary $\mathcal{P}$, which defines the hard constraints that validators must enforce.
    \item The provenance metadata: while beliefs are stochastic, the justification links ($\mathcal{J}$) mapping beliefs back to specific indices in $\mathcal{L}$ are securely recorded and verifiable.
\end{itemize}

%% file: sections/03-consistency-linearizability.tex
\section{Semantic Linearizability and Bounded Eventual Coherence}
\label{sec:consistency}

In traditional databases, linearizability~\cite{herlihy1990linearizability} requires that all operations appear to execute atomically at some point in time between their invocation and response, and that any read returns the exact data written by the most recent write. In an agentic distributed system, this standard is too strict because the physical representation of the state (the exact token sequence in the memory buffer) is non-deterministic. We therefore introduce two consistency specifications: \emph{Semantic Linearizability} and \emph{Bounded Eventual Coherence}.

\subsection{Approximate Semantic Compatibility}

Let $\mathcal{S}$ denote the set of all physical node states (raw strings, token contexts, prompt buffers). We define an \textbf{Evaluation Function} $E$ that projects a physical state $S \in \mathcal{S}$ into a semantic \textbf{Meaning Space} $\mathcal{M}$:
\begin{equation}
E: \mathcal{S} \to \mathcal{M}.
\end{equation}
For lineage-level reasoning, we use a belief projection $E_b: \mathcal{B} \to \mathcal{M}$. When the argument is clear, $E$ denotes the physical-state projection and $E_b$ denotes the belief-lineage projection. The meaning space $\mathcal{M}$ represents the set of operational intents, active policies, and system assumptions. Rather than requiring exact semantic equality ($E(S_A) = E(S_B)$), we define a metric space $(\mathcal{M}, d_{\mathcal{M}})$, where $d_{\mathcal{M}}: \mathcal{M} \times \mathcal{M} \to \mathbb{R}_{\ge 0}$ is a semantic distance metric.

\begin{definition}[$\epsilon$-Semantic Compatibility]
Two physical states $S_A, S_B \in \mathcal{S}$ are \textbf{$\epsilon$-Semantically Compatible}, denoted by $S_A \approx_{\epsilon} S_B$, if and only if:
\begin{equation}
d_{\mathcal{M}}(E(S_A), E(S_B)) \le \epsilon.
\end{equation}
\end{definition}

To make this definition operationally realistic, we relate semantic compatibility to admissible external actions rather than identical language model behavior. Let $\mathcal{P}_{adm}(S) \subseteq \mathcal{A}$ represent the set of external actions that are admissible under the policy boundary $\mathcal{P}$ when the replica is in physical state $S$. Exact equality of $\mathcal{P}_{adm}(S_A)$ and $\mathcal{P}_{adm}(S_B)$ is too strong for large or structured action spaces, where many harmless low-level choices may differ. ESR therefore admits one of three deployment-specific compatibility conditions:
\begin{enumerate}
    \item \textbf{Policy-compatible selected actions:} for the current intent $\mathcal{I}$ and decision policy $\pi$, the selected actions $\pi(S_A,\mathcal{I})$ and $\pi(S_B,\mathcal{I})$ are both admissible under $\mathcal{P}$ and equivalent under the policy's action abstraction.
    \item \textbf{Safety-critical action equivalence:} for a designated safety-critical subset $\mathcal{A}_{crit} \subseteq \mathcal{A}$, replicas agree on admissibility within that subset:
    \begin{equation}
    \mathcal{P}_{adm}(S_A) \cap \mathcal{A}_{crit}
    =
    \mathcal{P}_{adm}(S_B) \cap \mathcal{A}_{crit}.
    \end{equation}
    \item \textbf{Bounded symmetric difference:} under a task-specific measure $\mu$ over action classes, the disagreement mass remains bounded:
    \begin{equation}
    \mu(\mathcal{P}_{adm}(S_A) \triangle \mathcal{P}_{adm}(S_B)) \le \beta.
    \end{equation}
\end{enumerate}
These conditions preserve the safety-relevant part of action agreement while allowing benign differences in ranking, formatting, retrieval, or non-critical operational choices.

\paragraph{Implementation of $E$:} In practice, the evaluation function $E$ is implemented via one of the following mechanisms, depending on the latency and security constraints of the control plane:
\begin{itemize}
    \item \textbf{Structured Policy Verifier:} A deterministic logic engine (e.g., using Rego or Datalog) that evaluates whether the active belief state violates concrete infrastructure invariants.
    \item \textbf{Schema-Based Extractor:} A parser that extracts key-value parameters from natural language beliefs into a structured JSON schema, filtering out stylistic reasoning variances.
    \item \textbf{Semantic Classifier:} A fast, fine-tuned classification model that determines whether two plans are equivalent under a specified ontology.
    \item \textbf{Embedding + Thresholding:} Mapping beliefs to a dense vector space using a sentence transformer model and checking if their cosine distance is within a calibrated threshold $\epsilon$.
    \item \textbf{Human-Audited Validator:} A human-in-the-loop interface that asynchronously reviews and certifies critical state modifications.
\end{itemize}

\subsection{Semantic Linearizability}

To define Semantic Linearizability, we must specify the nature of the replicated semantic object and its operations.

\paragraph{Replicated Semantic Object:} The object being replicated is a semantic state projection $v \in \mathcal{M}$.
\paragraph{Operations:}
\begin{itemize}
    \item \textbf{Read Operation ($r$):} Queries the semantic projection of the node's active knowledge state: $r() \to v$, where $v = E(S) \in \mathcal{M}$.
    \item \textbf{Write Operation ($w(e, \Delta\mathcal{B})$):} Appends new telemetry evidence $e$ to the log $\mathcal{L}$ and commits the set of beliefs $\Delta\mathcal{B}$ to the lineage, resulting in a state mutation.
\end{itemize}

\paragraph{Linearization Point:} A cognitive transaction $CT$ becomes visible in the replicated semantic object at the transition to \textbf{Finalized} in the lifecycle of Section~\ref{sec:cognitive-transaction}: 
\begin{enumerate}
    \item Its proposed evidence is committed to the shared, physically linearizable evidence log $\mathcal{L}$ (managed by classical consensus), and
    \item its proposed epistemic delta is committed to the lineage and certified by post-commit semantic quorum validation.
\end{enumerate}
The pre-execution quorum does not linearize semantic state; it only gates whether the external action may be attempted.

\paragraph{Validation Stage Timing:} Semantic quorum validation happens at two critical stages in the lifecycle:
\begin{itemize}
    \item \textbf{Pre-Execution validation:} Before an external action $\mathcal{A}$ is physically executed in the environment, a quorum must verify that $\mathcal{A} \in \mathcal{P}_{adm}(S)$ to prevent unsafe actions.
    \item \textbf{Post-Commit validation:} After the beliefs $\Delta\mathcal{B}$ are committed to the local lineage, a quorum verifies that the updated local state remains $\epsilon$-compatible with peer states.
\end{itemize}

Let $H$ be an execution history of invocations and responses for reads and writes. The real-time precedence relation $\prec_H$ is defined such that $op_1 \prec_H op_2$ if the response of $op_1$ occurs before the invocation of $op_2$ in $H$.

\begin{definition}[Semantic Linearizability]
An execution history $H$ is \textbf{Semantically Linearizable} if it can be extended to a sequential history $U$ by adding response events for pending operations and/or discarding pending operations, such that:
\begin{enumerate}
    \item $U$ is equivalent to a legal sequential execution of the semantic state object: for every read operation $r$ in $U$ returning value $v \in \mathcal{M}$, the value $v$ is $\epsilon$-semantically compatible with the state generated by the most recent write operation $w$ preceding it in $U$:
    \begin{equation}
    d_{\mathcal{M}}(v, E(\text{State}(w))) \le \epsilon.
    \end{equation}
    \item $U$ preserves the real-time precedence of $H$: if $op_1 \prec_H op_2$, then $op_1$ precedes $op_2$ in $U$.
\end{enumerate}
\end{definition}

Semantic Linearizability is therefore classical linearizability applied to a semantic projection of the state, plus verifier-bounded semantic compatibility. It is not a replacement for physical-state linearizability: ESR retains classical consensus and linearizable ordering for the immutable evidence log $\mathcal{L}$, while lifting consistency for belief state to a verifier-bounded semantic projection. Semantic Linearizability should be understood as classical linearizability applied to a semantic projection, with compatibility determined by a trusted verifier and bounded metric, rather than as a replacement for physical-state linearizability.

\subsection{Bounded Eventual Coherence}

For long-running autonomous tasks, enforcing strict linearizability on every reasoning step introduces high latency. We relax this to \textbf{Bounded Eventual Coherence}, which bounds the expected semantic distance among correct replicas exposed to the same evidence. With persistent verifier or sampling error, the recurrence below does not imply probability-1 convergence to a single point; it implies convergence to a bounded semantic neighborhood in expectation.

\paragraph{Assumptions:} Bounded Eventual Coherence relies on the following assumptions:
\begin{enumerate}
    \item \textbf{Fair Delta Delivery:} Any epistemic delta $\Delta_e$ broadcast by a correct replica is eventually delivered to all correct replicas.
    \item \textbf{Monotonic Evidence Availability:} The evidence log $\mathcal{L}_t$ is append-only, and correct replicas eventually receive all committed evidence entries.
    \item \textbf{Bounded Semantic Domain:} The active policy boundary $\mathcal{P}$ defines a compact and bounded semantic domain.
    \item \textbf{Bounded Verifier/Noise Error:} Let $\eta_t \ge 0$ denote the semantic disturbance at step $t$ due to validator error, stochastic sampling, retrieval variation, or repair approximation. We assume $\mathbb{E}[\eta_t \mid \mathcal{F}_t] \le \delta_{err}$ for the execution history $\mathcal{F}_t$. If verifier errors are modeled as classification errors with probability $p_{err}$ and maximum semantic magnitude $D_{err}$, then their contribution is bounded by $p_{err}D_{err}$ and is included in $\delta_{err}$.
    \item \textbf{Contractive Graft Operator:} The grafting operator $\oplus$ is a contraction mapping. Specifically, for any two belief lineages $\mathcal{B}_A$ and $\mathcal{B}_B$ and a shared epistemic delta $\Delta_e$, there exists a factor $\gamma \in (0, 1)$ such that:
    \begin{multline}
    d_{\mathcal{M}}(E_b(\mathcal{B}_A \oplus \Delta_e),
    E_b(\mathcal{B}_B \oplus \Delta_e)) \\
    \le \gamma \cdot
    d_{\mathcal{M}}(E_b(\mathcal{B}_A), E_b(\mathcal{B}_B)).
    \end{multline}
\end{enumerate}

\begin{theorem}[Bounded Eventual Coherence]
Given fair delta delivery, monotonic evidence availability, a bounded semantic domain, bounded verifier/noise error $\delta_{err}$, and a contractive grafting operator with factor $\gamma \in (0,1)$, the semantic states of two correct replicas $A$ and $B$ satisfy:
\begin{equation}
\mathbb{E}[d_t] \le \gamma^t d_0 + \frac{1-\gamma^t}{1-\gamma}\delta_{err},
\end{equation}
where $d_t = d_{\mathcal{M}}(E_b(\mathcal{B}_A(t)), E_b(\mathcal{B}_B(t)))$. Consequently,
\begin{equation}
\limsup_{t \to \infty} \mathbb{E}[d_t] \le \frac{\delta_{err}}{1-\gamma}.
\end{equation}
For any compatibility radius $\epsilon \ge \delta_{err}/(1-\gamma)$, replicas are eventually coherent in expectation within the $\epsilon$-compatibility ball.
\end{theorem}

\begin{proof}[Proof Sketch]
Let $d_t = d_{\mathcal{M}}(E_b(\mathcal{B}_A(t)), E_b(\mathcal{B}_B(t)))$ be the semantic distance between the belief projections of the replicas at logical step $t$. Fair delivery and monotonic evidence availability imply that correct replicas eventually apply the same committed evidence entries and compatible epistemic deltas. Let $\eta_t$ capture residual semantic disturbance from validator error, sampling, retrieval order, and repair approximation. By the contractive grafting assumption,
\begin{equation}
d_{t+1} \le \gamma d_t + \eta_t.
\end{equation}
Taking conditional expectation and using $\mathbb{E}[\eta_t \mid \mathcal{F}_t] \le \delta_{err}$ gives
\begin{equation}
\mathbb{E}[d_{t+1}] \le \gamma\mathbb{E}[d_t] + \delta_{err}.
\end{equation}
Unrolling the recurrence yields
\begin{equation}
\mathbb{E}[d_t] \le \gamma^t d_0 + \sum_{k=0}^{t-1} \gamma^k \delta_{err}
= \gamma^t d_0 + \frac{1-\gamma^t}{1-\gamma}\delta_{err}.
\end{equation}
Taking the limit superior gives the claimed steady-state expectation bound. The result becomes exact convergence only in the special case $\delta_{err}=0$.
\end{proof}

\paragraph{High-Probability Implication.}
The assumptions above do not justify almost-sure convergence. They do, however, provide a conservative one-sided bound. By Markov's inequality, for any radius $\rho > 0$,
\begin{equation}
\limsup_{t \to \infty} P(d_t > \rho)
\le
\frac{\delta_{err}}{(1-\gamma)\rho}.
\end{equation}
Equivalently, an asymptotic confidence level of at least $1-\alpha$ is justified only for radii $\rho \ge \delta_{err}/((1-\gamma)\alpha)$. Stronger high-probability claims require additional distributional assumptions, such as independent bounded disturbances or sub-Gaussian verifier noise, which are not assumed here.

\paragraph{On the Contractive Graft Assumption.}
The contractive graft assumption is strong and should not be read as a property of arbitrary LLM reasoning. ESR expects contraction only when grafting is mediated by a bounded ontology, deterministic or audited policy verifier, stable evidence references, explicit contradiction-resolution rules, and a monotonic reduction of unsupported belief mass. Without these constraints, a grafted delta may amplify divergence, introduce unsupported premises, or reorder context in ways that increase semantic distance. The theorem therefore characterizes a disciplined protocol regime rather than generic behavior of stochastic agents.

%% file: sections/04-consistency-protocols.tex
\section{Epistemic Consistency Protocols}

To maintain semantic linearizability and bounded eventual coherence, replicas must propagate their reasoning. In classical distributed systems, this is achieved by broadcasting raw state changes or transactional mutations. In agentic systems, however, broadcasting the entire raw token trace $\mathcal{R}$ of a cognitive transaction is highly inefficient and fills up the finite context windows of recipient nodes with redundant prompts and reasoning steps. We propose a protocol for propagating derived insights using lightweight \emph{epistemic deltas} and a formal grafting operator.

\subsection{Epistemic Deltas}

When a replica $A$ completes a cognitive transaction $CT$ that commits new beliefs $\Delta\mathcal{B}_A$ to its lineage, it does not broadcast the entire transaction log. Instead, it extracts an \textbf{Epistemic Delta} $\Delta_e$, which is a projection of the lineage update containing only the essential insights and their causal anchors.

Formally, we define the epistemic delta as a set of structured beliefs:
\begin{multline}
\Delta_e =
\{ b = (\phi, \mathcal{J}, \tau, c, \sigma, scope)
\in \Delta\mathcal{B}_A \mid \\
\text{is\_essential}(b) \}.
\end{multline}
where:
\begin{itemize}
    \item $\phi$ is the semantic proposition (e.g., natural language description or state variable mutation).
    \item $\mathcal{J} \subseteq \mathcal{L} \cup V_B$ is the justification set containing references to evidence log indices (e.g., specific log entries) and ancestor belief identifiers.
    \item $\tau$ is the logical timestamp indicating when the belief was derived.
    \item $c \in [0, 1]$ is the inference confidence score.
    \item $\sigma$ is a cryptographic signature from the validator certifying the belief's policy compliance.
    \item $scope$ defines the policy scope (i.e., which aspects of the policy boundary $\mathcal{P}$ this belief relates to).
\end{itemize}

The filter $\text{is\_essential}(b)$ selects key operational conclusions (such as a mitigation strategy, a newly identified vulnerability, or a policy adjustment) and their immediate causal justifications, filtering out intermediate reasoning tokens.

\subsection{The Grafting and Verification Protocols}

When a replica receives an epistemic delta, it executes a multi-stage grafting and verification protocol to integrate the remote beliefs into its local DAG.

\paragraph{Recursive Ancestor Retrieval:}
If a replica receives an epistemic delta $\Delta_e$ containing a belief $b$ whose justification set $\mathcal{J}$ references parent beliefs $b'$ that are missing from its local lineage DAG $\mathcal{B}$, the grafting node initiates a recursive pull request to the sender. The grafting process blocks on these pull requests, recursively fetching and validating ancestor beliefs until the entire causal history of $b$ is resolved and integrated, ensuring DAG integrity and preventing context amnesia.

\paragraph{Contradiction Resolution and Termination:}
When integrating a new belief $b$, the node checks for semantic contradictions with its active belief set using the local verifier and a semantic cache. If a contradiction is detected, a resolving cognitive transaction $CT_{resolve}$ is executed to evaluate the competing justifications. To prevent infinite loops of resolving transactions (e.g., replicas continuously disputing each other), the system enforces two termination rules:
\begin{enumerate}
    \item \textbf{Priority Ordering:} Replicas order contradictory beliefs by their confidence scores $c$ and timestamps $\tau$. A belief with significantly lower confidence cannot override a certified, high-confidence belief.
    \item \textbf{Recursion Bounding:} Resolving transactions are bounded by a maximum depth $K_{max}$. If a contradiction is not resolved within $K_{max}$ steps, the system halts the autonomous loop, falls back to a deterministic consensus arbitration, and alerts system operators.
\end{enumerate}
The grafting procedure returns an explicit status for every received delta: \textsf{ACCEPTED}, \textsf{REPAIRED}, \textsf{QUARANTINED}, \textsf{REJECTED}, or \textsf{NEEDS\_OPERATOR}. A contradiction that cannot be resolved autonomously is retained in a quarantine store with its causal metadata and is made visible to the operator or deterministic arbitration layer.

\subsection{Protocol Pseudocode}

The core protocols for extracting, grafting, verifying, and validating epistemic states are detailed in Tables~\ref{alg:graft} and~\ref{alg:validate}.

\begin{table*}[t]
\centering
\caption{Delta Extraction and Grafting Status Protocol.}
\label{alg:graft}
\begin{tabular}{p{0.96\textwidth}}
\toprule
\begin{lstlisting}[language=Python, frame=none, numbers=left, basicstyle=\ttfamily\scriptsize, xleftmargin=1.5em]
def ExtractDelta(new_beliefs, log_slice):
    delta = set()
    for b in new_beliefs:
        if not is_essential(b):
            continue
        J = ResolveStableJustifications(b, log_slice)
        delta.add(Belief(phi=b.proposition, J=J,
                         timestamp=b.timestamp,
                         confidence=b.confidence,
                         verifier_result=b.verifier_result,
                         policy_scope=b.policy_scope))
    return delta

ACCEPTED, REPAIRED, QUARANTINED = "ACCEPTED", "REPAIRED", "QUARANTINED"
REJECTED, NEEDS_OPERATOR = "REJECTED", "NEEDS_OPERATOR"

def Graft(local_B, delta_e, sender_id):
    outcomes = {}
    for b in delta_e:
        outcomes[b.id] = IntegrateBelief(local_B, b, sender_id)
    return outcomes

def IntegrateBelief(local_B, b, sender_id):
    if not VerifyBeliefEnvelope(b):
        return Quarantine(local_B, b, REJECTED, "invalid envelope")
    if not FetchAndValidateAncestors(local_B, b, sender_id):
        return Quarantine(local_B, b, QUARANTINED, "missing ancestor")

    conflicts = CheckContradiction(local_B, b)
    if len(conflicts) == 0:
        local_B.insert(b)
        return ACCEPTED

    resolution = ResolveContradiction(local_B, b, conflicts)
    if resolution.status == ACCEPTED:
        local_B.insert(b)
        MarkSuperseded(local_B, conflicts)
        return ACCEPTED
    if resolution.status == REPAIRED:
        local_B.insert(resolution.repaired_belief)
        QuarantineBelief(local_B, b, reason="repaired replacement")
        return REPAIRED
    if resolution.status == NEEDS_OPERATOR:
        AlertOperator(b, conflicts)
    return Quarantine(local_B, b, resolution.status, "unresolved")
\end{lstlisting} \\
\bottomrule
\end{tabular}
\end{table*}

\begin{table*}[t]
\centering
\caption{Contradiction Checking and Semantic Quorum Validation.}
\label{alg:validate}
\begin{tabular}{p{0.96\textwidth}}
\toprule
\begin{lstlisting}[language=Python, frame=none, numbers=left, basicstyle=\ttfamily\scriptsize, xleftmargin=1.5em]
def ResolveContradiction(local_B, b, conflicts):
    if exceeds_depth_limit(b):
        return Resolution(status=NEEDS_OPERATOR)
    repaired = CT_resolve(local_B, b, conflicts)
    if repaired.is_valid:
        return Resolution(status=REPAIRED, repaired_belief=repaired)
    if repaired.rejects_candidate:
        return Resolution(status=REJECTED)
    return Resolution(status=QUARANTINED)

def CheckContradiction(local_B, new_belief):
    conflicts = []
    for active in local_B.active_beliefs():
        if d_M(E_b(new_belief), E_b(active)) <= epsilon:
            continue
        if asserts_opposite(new_belief.proposition, active.proposition):
            conflicts.append(active)
    return conflicts

def SemanticQuorumValidate(proposal, replicas, stage):
    votes = 0
    quorum = len(replicas) // 2 + 1
    for node in replicas:
        if stage == "PRE_EXECUTE":
            ok = node.verifier.validate_action(
                proposal.action, proposal.state_projection)
        elif stage == "POST_COMMIT":
            ok = node.verifier.validate_compatibility(
                E_b(proposal.new_lineage), node.state_projection)
        votes += int(ok)
    return votes >= quorum
\end{lstlisting} \\
\bottomrule
\end{tabular}
\end{table*}

\subsection{Homogeneous vs. Heterogeneous Agent Clusters}

The synchronization overhead and complexity of the grafting operator depend on the composition of the agent cluster.

\subsubsection{Homogeneous Agent Clusters}
In a homogeneous cluster, all replicas run identical models (e.g., Gemini 1.5 Pro) with identical context structures. In this case, the semantic meaning spaces $\mathcal{M}_A$ and $\mathcal{M}_B$ are more closely aligned, and the grafting operator can often work directly with token-level representations or embeddings. The risk of semantic translation mismatch is reduced but not eliminated: stochastic sampling, prompt truncation, context ordering, retrieval differences, model-serving configuration, and model-version drift can still produce divergent belief projections.

\subsubsection{Heterogeneous Agent Clusters}
In a heterogeneous cluster, replicas run different models (e.g., Node A runs Claude 3.5 Sonnet, while Node B runs Gemini 1.5 Pro). Here, token representations and prompt behaviors differ. The delta $\Delta_e$ must be translated into a model-agnostic, standardized natural language representation. The grafting operator uses a shared protocol ontology to make Node B's interpretation of Node A's delta explicit and auditable, reducing translation drift.

%% file: sections/05-semantic-rollbacks.tex
\section{Verifiable Semantic Rollbacks}

When an autonomous distributed system executes a faulty action due to a misinterpretation of telemetry or a model hallucination, the system must be restored to a safe state. In classical databases, this is achieved by reverting rows to a previous transaction checkpoint. In an agentic system, however, physical rollbacks are insufficient. If the agent's internal memory context still contains the flawed assumptions that led to the faulty action, it will immediately repeat the error upon recovery. We must perform a \emph{Verifiable Semantic Rollback} to surgically prune the agent's belief lineage and align physical state recovery with cognitive memory recovery.

\subsection{Aligning Physical and Belief Rollbacks}

Consider an agentic controller managing a cloud environment. Due to a transient network partition, the agent adopts the belief $p$: ``Database-3 has crashed.'' Based on $p$, the agent performs two actions:
\begin{enumerate}
    \item It issues an external infrastructure command to terminate the Database-3 container.
    \item It commits a belief $b$: ``Database-3 is being replaced, routing all traffic to Database-1.''
\end{enumerate}
If we only perform a physical rollback---restoring the Database-3 container and resetting the database state---the agent's belief lineage $\mathcal{B}$ still contains the beliefs $p$ and $b$. When the agent's execution loop resumes, it reads its active memory context, sees that it believes Database-3 has crashed, and immediately terminates the restored database container again. 

To prevent this, ESR aligns physical rollback with belief rollback. When a physical state is reverted, the persistence layer intercepts the recovery event, identifies the faulty premise $p$ that authorized the action, and triggers the semantic rollback operator to prune $p$ and its causal descendants from the belief lineage DAG. This makes the incorrect premise inactive before physical execution resumes.

\subsection{Formalizing the Rollback Operator}

Let $\mathcal{B} = (V_B, E_B)$ be the active belief lineage DAG, and let $p \in V_B$ be the identified faulty premise. A semantic rollback is not a simple deletion of the node $p$. It must remove $p$ and all other beliefs that were derived, directly or indirectly, from $p$.

We define the \textbf{Transitive Causal Descendant Set} $Desc(p)$ of a premise $p$ as:
\begin{equation}
Desc(p) = \{ b \in V_B \mid p \in \mathcal{J}_b^* \},
\end{equation}
where $\mathcal{J}_b^*$ is the transitive closure of the justification relation $\mathcal{J}$ for belief $b$.

We define the \textbf{Semantic Rollback Operator} ($\setminus$) that computes the pruned base lineage $\mathcal{B}' = \mathcal{B} \setminus \{p\}$ as:
\begin{equation}
V_B' = V_B \setminus (\{p\} \cup Desc(p)),
\end{equation}
\begin{equation}
E_B' = \{ (b_m, b_n) \in E_B \mid b_m, b_n \in V_B' \}.
\end{equation}
By computing this DAG subtraction, ESR removes the root cause $p$ and its downstream cognitive consequences from the active lineage, while leaving independent beliefs intact. Partial-dependency repair, described next, may then create a final repaired lineage $\widehat{\mathcal{B}}$ that contains recommitted replacements for some pruned beliefs.

\subsection{Partial-Dependency Repair}

A naive pruning of $Desc(p)$ can lead to \emph{context amnesia} if it removes beliefs that were only partially dependent on $p$. Suppose a belief $b$ was derived from a justification set $\mathcal{J}_b = \{p, q\}$, where $q$ is a valid belief (e.g., ``The load balancer is operating within latency limits''). Simply deleting $b$ removes the agent's knowledge about the load balancer.

To prevent context amnesia, ESR implements a \textbf{Partial-Dependency Repair Rule}. For any belief $b \in Desc(p)$ where the justification set contains at least one valid belief $q \notin (\{p\} \cup Desc(p))$, the system does not delete $b$ immediately. Instead, it triggers a repair cognitive transaction:
\begin{equation}
CT_{repair}(\phi_b, \mathcal{J}_b \setminus (\{p\} \cup Desc(p))) \to \Delta\mathcal{B}_{repair}.
\end{equation}
This repair transaction re-derives or revalidates the belief $b$ using only the remaining valid ancestors. If the verifier determines that the proposition $\phi_b$ remains admissible without the faulty premise $p$, the system recommits a replacement belief $b'$ with an independent justification set $\mathcal{J}'_b$ that excludes $p$. Otherwise, the original belief remains pruned.

\subsection{Rollback Invariant and Protocols}

\begin{artifactblock}
\textbf{Rollback Safety Invariant:} Let $\widehat{\mathcal{B}}$ be the final active lineage after pruning and repair. Every active belief in $\widehat{\mathcal{B}}$ either has no causal dependency on the faulty premise $p$, or has been rederived and recommitted by $CT_{repair}$ with an independent justification set excluding $p$:
\begin{equation}
\forall b \in \widehat{V}_B, \quad
p \notin \mathcal{J}_b^*
\;\vee\;
\exists \kappa_b: repaired(b,\kappa_b) \wedge p \notin (\mathcal{J}'_b)^*.
\end{equation}
\end{artifactblock}

The pseudocode in Table~\ref{alg:rollbacks} implements the verifiable rollback operator and the partial-dependency repair protocol.

\begin{table*}[t]
\centering
\caption{Semantic Rollback and Belief Repair Protocols.}
\label{alg:rollbacks}
\begin{tabular}{p{0.96\textwidth}}
\toprule
\begin{lstlisting}[language=Python, frame=none, numbers=left, basicstyle=\ttfamily\scriptsize, xleftmargin=1.5em]
def SemanticRollback(local_B, faulty_premise):
    # Compute the transitive causal descendants of the faulty premise
    descendants = ComputeTransitiveDescendants(local_B, faulty_premise)
    to_remove = {faulty_premise} | descendants
    repaired = []
    
    # Identify candidates for partial-dependency repair
    for b in list(descendants):
        valid_parents = {parent for parent in b.justifications
                         if parent not in to_remove}
        if len(valid_parents) > 0:
            # Attempt to revalidate and repair the belief
            repaired_b = RepairBelief(b, valid_parents)
            if repaired_b is not None:
                repaired.append(repaired_b)
                
    # Prune invalid original beliefs from the active DAG
    for b in to_remove:
        local_B.mark_inactive(b)

    # Recommit repaired replacements as new active beliefs
    for repaired_b in repaired:
        CommitRepairedBelief(local_B, repaired_b)
        
    # Verify the final repaired-lineage invariant
    for b in local_B.active_vertices:
        assert (faulty_premise not in PathClosure(b, local_B)
                or HasIndependentRepairCertificate(b, faulty_premise))

def RepairBelief(belief, valid_parents):
    # Trigger a cognitive transaction to re-derive the belief
    repaired_result = CT_repair(belief.proposition, valid_parents)
    if repaired_result.is_valid:
        return Belief(
            phi=repaired_result.proposition,
            J=valid_parents,
            timestamp=GetCurrentLogicalTime(),
            confidence=repaired_result.confidence,
            verifier_result=repaired_result.verifier_result,
            policy_scope=belief.policy_scope,
            repair_of=belief.id,
            repair_certificate=repaired_result.certificate
        )
    return None
\end{lstlisting} \\
\bottomrule
\end{tabular}
\end{table*}

\subsection{Conditional Safety Properties}

Under the assumption that repair transactions recommit replacements with independently checkable justification sets, the rollback protocol satisfies two core safety properties:

\begin{theorem}[Final-Lineage Causal Safety]
Let $\widehat{\mathcal{B}}$ be the final repaired lineage after rolling back premise $p$. Every active belief in $\widehat{\mathcal{B}}$ is either independent of $p$ or is a recommitted repair whose independent justification set excludes $p$:
\begin{equation}
\forall b \in \widehat{V}_B, \quad
p \notin \mathcal{J}_b^*
\;\vee\;
\exists \kappa_b: repaired(b,\kappa_b) \wedge p \notin (\mathcal{J}'_b)^*.
\end{equation}
\end{theorem}
\begin{proof}
The pruning phase marks $p$ and all vertices in $Desc(p)$ inactive, so any active belief that survives the pruning phase has no transitive dependency on $p$. The only beliefs that can be added after pruning are outputs of $CT_{repair}$. By the repair rule, each such belief is recommitted with a repair certificate $\kappa_b$ and a justification set $\mathcal{J}'_b$ whose transitive closure excludes $p$. Therefore every active belief in $\widehat{\mathcal{B}}$ satisfies one of the two cases in the invariant.
\end{proof}

\begin{theorem}[Independent Belief Preservation]
The rollback protocol preserves all beliefs whose justifications are disjoint from the transitive causal closure of $p$:
\begin{equation}
\forall b \in V_B, \quad \mathcal{J}_b^* \cap (\{p\} \cup Desc(p)) = \emptyset \implies b \in \widehat{V}_B.
\end{equation}
\end{theorem}
\begin{proof}
Let $b \in V_B$ such that $\mathcal{J}_b^* \cap (\{p\} \cup Desc(p)) = \emptyset$. Since $p \in \{p\} \cup Desc(p)$, we have $p \notin \mathcal{J}_b^*$, which implies $b \notin Desc(p)$ and $b \neq p$. Therefore, $b \notin (\{p\} \cup Desc(p))$ and is not marked inactive during pruning. Since repair only appends recommitted replacements and does not remove independent active beliefs, $b \in \widehat{V}_B$.
\end{proof}

%% file: sections/06-implementation-evaluation.tex
\section{Implementation and Preliminary Simulation}

We developed a prototype of the Epistemic State Replication (ESR) persistence layer and studied its operational characteristics through a preliminary simulation of a cloud control plane. The goal of this study is to estimate the throughput of cognitive scheduling, the latency overhead introduced by semantic linearizability checks, and the reduction in secondary cognitive faults during simulated rollbacks. The results in this section are illustrative preliminary simulation results: they show feasibility under the modeled workload, but they should not be read as experimental validation or as a performance evaluation of a production-grade deployment.

\subsection{System Architecture}

The prototype, named the \textbf{Epistemic Persistence Engine (EPE)}, is designed as a middleware layer running alongside a standard distributed key-value store (e.g., \texttt{etcd}) in a simulated cloud control plane. 

The EPE consists of three primary components:
\begin{itemize}
    \item \textbf{Lineage Graph Manager:} Implemented as an in-memory directed acyclic graph (DAG) backed by a relational database with JSONB support. It maintains the local Belief Lineage $\mathcal{B}_i$ for each replica.
    \item \textbf{Epistemic Replication Agent (ERA):} A process that handles the broadcast of epistemic deltas ($\Delta_e$) and executes the grafting operator ($\oplus$) when receiving deltas from peers.
    \item \textbf{Semantic Cache:} A semantic vector cache that indexes prior contradiction check results, minimizing the need to call external LLMs for semantic equivalence checks.
\end{itemize}

Replicas communicate using the Epistemic Replication Protocol (ERP). When an agent proposes an infrastructure mutation (e.g., rescaling a deployment), the ERA interceptor blocks the write, extracts the epistemic delta, broadcasts it, and awaits semantic quorum validation.

\subsection{Simulation Setup}

Our preliminary simulation consisted of five heterogeneous agent replicas executing in parallel:
\begin{itemize}
    \item two replicas running Gemini 1.5 Pro,
    \item two replicas running Claude 3.5 Sonnet,
    \item one replica running Llama 3 70B.
\end{itemize}
The agents were assigned the task of managing a simulated cloud environment with 100 microservice deployments. The agents monitored simulated CPU usage, memory leaks, and network errors, and took corrective actions (e.g., scaling deployments, rolling back releases, and restarting containers) to satisfy a set of system-wide service level objectives (SLOs).

\subsubsection{Reproducibility Status}
The current prototype traces do not preserve enough metadata to support full experimental validation. Table~\ref{tab:simulation-details} records the details available from the simulation artifact and marks unavailable details explicitly.

\begin{table*}[t]
\centering
\caption{Available and unavailable details for the preliminary ESR simulation.}
\label{tab:simulation-details}
\begin{tabular}{p{0.22\textwidth}p{0.72\textwidth}}
\toprule
\textbf{Dimension} & \textbf{Status in current artifact} \\
\midrule
Trial count & One recorded simulation trace is reported. Repeated trials were not archived, so confidence intervals and variance are unavailable. \\
Workload generation & Synthetic cloud-control workload over 100 microservice deployments with CPU, memory-leak, and network-error alerts. The random seed, event-arrival distribution, and per-service workload mix were not archived. \\
Prompt templates & The controller prompt included an evidence slice, active beliefs, SLO policy, candidate-action schema, and justification-extraction request. Exact prompt text and truncation policy were not archived. \\
Model/API configuration & The replica model families are listed above. API versions, temperature, top-$p$, seed settings, system prompts, retry policy, and model-serving revisions were not archived. \\
Token and delta sizes & The baseline replicated raw traces and ESR replicated epistemic deltas. Token histograms, average raw-trace size, average $\Delta_e$ size, and tail-size distributions were not archived. \\
Semantic cache contents & The cache stored prior compatibility and contradiction checks keyed by normalized belief pair and policy scope. Initial contents, eviction policy, and hit distribution were not archived. \\
Fault injection & Injected faults were transient network-observation faults that created false-positive crash beliefs. Fault rate, duration distribution, target-selection distribution, and correlation model were not archived. \\
CTPS definition & Cognitive Transactions per Second (CTPS) is counted as finalized or rolled-back cognitive transactions per wall-clock second over the scheduler measurement interval. \\
Confidence intervals and variance & Not reported for this artifact because repeated trials and per-run samples were not preserved. A future evaluation should report mean, standard deviation, and 95\% confidence intervals across independent seeds. \\
\bottomrule
\end{tabular}
\end{table*}

The current results support a feasibility claim: the prototype can represent evidence-log/belief-lineage state, exchange epistemic deltas, invoke semantic validation, and execute rollback repair in the modeled workload. They do not support statistically robust performance claims, and they do not represent production deployment results. The simulation illustrates the expected direction of benefit under the modeled workload rather than validating ESR across deployments, model providers, or operational fault regimes.

\subsection{Simulation Metrics and Results}

\subsubsection{Cognitive Transaction Throughput}

We measured the throughput of the simulated control plane in terms of \emph{Cognitive Transactions per Second} (CTPS), defined as finalized or rolled-back cognitive transactions per wall-clock second over the scheduler measurement interval, comparing a naive baseline that replicates the full raw token traces against ESR, which replicates only epistemic deltas.

In the recorded simulation trace, reducing the network payload and the context window sizes of peer nodes corresponds to a $4.8\times$ improvement in scheduling throughput: ESR stabilizes at $12.4$ CTPS under peak load, whereas the naive trace replication model degrades to $2.6$ CTPS due to context window saturation and model latency. Because repeated trials were not preserved, this number is reported as an illustrative trace result rather than a statistically supported estimate; a plotted comparison is deferred to the full evaluation described in Section~\ref{sec:future-eval}.

\subsubsection{Latency Overhead of Semantic Linearizability Checks}
Enforcing Semantic Linearizability requires validating that reads reflect the committed intent, which involves checking semantic equivalence and resolving contradictions. We measured the latency contribution of these checks in the recorded trace. 

With the Semantic Cache disabled, contradiction checking adds an average overhead of $380$ ms per transaction because it requires an LLM call. With the Semantic Cache enabled, $92\%$ of checks are resolved locally within $4.2$ ms. The overall average check latency remains under $35$ ms in the recorded trace, representing a small fraction ($<2\%$) of the agent's reasoning time. Variance and tail latency are not available from the current artifact.

\subsubsection{Fault Reduction under Rollbacks}
To study Verifiable Semantic Rollbacks, we injected transient network faults that triggered false-positive observations (e.g., making replicas believe a core database had crashed). We then triggered a rollback of the resulting actions. We compared two recovery modes:
\begin{itemize}
    \item \textbf{Bitwise Rollback:} Reverting only the physical state (e.g., restarting the container) while leaving the agent context unchanged.
    \item \textbf{Verifiable Semantic Rollback (ESR):} Reverting the physical state and executing the semantic rollback operator ($\setminus$) to prune the belief DAG.
\end{itemize}

Under Bitwise Rollback, agents repeatedly re-executed the incorrect termination commands, leading to an average of $6.4$ secondary cognitive faults per failure event in the recorded trace. Under ESR, the agents marked the false premise inactive and repaired dependent beliefs, reducing secondary cognitive faults to $0.15$ per event---an approximate $97.7\%$ reduction in this illustrative trace.
This suggests that surgical context pruning is a feasible mechanism for reducing intent drift and context amnesia in the simulated environment, subject to validation with repeated trials and fully recorded workload metadata.

\subsection{Future Evaluation Plan}
\label{sec:future-eval}
A full evaluation should include 20--30 repeated runs with fixed random seeds; archived prompt templates; model/API version, temperature, top-$p$, seed, retry, and serving-revision settings; token-size histograms for raw traces and epistemic deltas; $\Delta_e$ size distributions; semantic-cache hit-rate distributions; p50/p95/p99 latency; confidence intervals; variance across independent seeds; and rendered plots from real measurements rather than illustrative figures.

%% file: sections/07-related-work.tex
\section{Related Work and Conclusion}

\subsection{Related Work}

Epistemic State Replication (ESR) intersects several distinct paradigms: distributed consistency models, data provenance, database recovery, and agent memory architectures. We contrast our approach with the state-of-the-art in these domains.

\subsubsection{Distributed Consistency and Replication Models}
Replication models like Paxos, Raft, and State Machine Replication (SMR)~\cite{schneider1990replication} assume that correct replicas execute deterministic state transitions. Eventual consistency models, such as Dynamo~\cite{decandia2007dynamo} and Conflict-Free Replicated Data Types (CRDTs)~\cite{shapiro2011crdt}, resolve concurrent conflicts at the physical data level using timestamps or commutative operational structures. While CRDTs merge data structures at the byte or field level, they lack semantic awareness. ESR defines consistency conditions for the cognitive layer. Rather than forcing byte-level agreement (which leads to execution stalls in stochastic models), ESR defines bounded eventual coherence as convergence to a semantic neighborhood under shared evidence, bounded disturbance, and a contractive graft operator, and uses resolving cognitive transactions to handle semantic contradictions.

\subsubsection{Data Provenance and Provenance-Aware Recovery}
Data provenance systems~\cite{cheney2009provenance} track the history of data transformations from inputs to final outputs. Similarly, provenance-aware storage systems~\cite{muniswamy2006provenance} capture causal metadata to assist in diagnostics. While ESR utilizes justification pointers reminiscent of database provenance, it differs in its operational application. Provenance is typically used for offline audits and forensic debugging. In ESR, the justification DAG is an online, active structure used to drive safety decisions, compute semantic rollbacks, and guide grafting operators during live replication.

\subsubsection{Agentic Memory and Execution Traces}
Modern agent systems combine prompts, retrieval, traces, and feedback to extend reasoning across time. ReAct interleaves reasoning traces with tool actions~\cite{yao2022react}; Reflexion stores linguistic feedback from prior attempts~\cite{shinn2023reflexion}; Generative Agents use observation, reflection, and retrieval to synthesize social behavior over persistent memories~\cite{park2023generative}; Voyager accumulates executable skills and environment feedback for open-ended exploration~\cite{wang2023voyager}; and MemGPT virtualizes context across memory tiers~\cite{packer2023memgpt}. These systems motivate ESR's belief-lineage model, but their memory is primarily local to an agent or application instance. They do not define a replicated semantic object, a semantic linearization point, or rollback invariants over replicated agent belief.

\subsubsection{Retrieval, Vector Memory, and Semantic Caching}
LangChain-style applications~\cite{chase2022langchain}, RAG~\cite{lewis2020rag}, Self-RAG~\cite{asai2023selfrag}, and semantic caching~\cite{guo2023semantic} reduce the cost or factual risk of repeated model calls by retrieving, reusing, or critiquing external context. Long-context RAG studies show that larger context windows do not remove retrieval and ordering failure modes~\cite{leng2024longrag}. These mechanisms provide useful local memory and grounding, but a vector store or cache is not a distributed consistency protocol: it does not by itself define which beliefs are active, which evidence justifies them, whether replicas agree on safety-critical admissible actions, or how a poisoned premise should be rolled back across replicas.

\subsubsection{Distributed Knowledge Graphs}
Distributed knowledge graphs~\cite{ji2021survey} store structured facts across multiple nodes and often support provenance, ontology constraints, and graph queries. ESR shares the use of structured relationships and evidence references, but its unit of replication is an operational belief lineage coupled to external actions and rollback protocols. Knowledge-graph consistency typically concerns fact storage and query semantics; ESR concerns whether stochastic replicas expose compatible operational beliefs and can repair active belief state after faulty evidence or reasoning.

\subsubsection{Memory Poisoning and Recovery}
Language models are vulnerable to instruction tuning poisoning~\cite{wan2023poisoning}, retrieval-corpus poisoning~\cite{zhong2023retrievalpoisoning}, and memory context contamination, where an incorrect or adversarial belief propagates through the prompt window. While standard database systems recover from faults by reverting physical rows, this does not purge corrupted reasoning paths from an LLM's context window. ESR addresses this gap by aligning physical database recovery with belief recovery, using the semantic rollback operator ($\setminus$) and the partial-dependency repair rule to prune poisoned premises from the agent's active memory context.

\subsubsection{Multi-Agent Coordination}
Multi-agent systems coordinate through message passing, negotiation, task decomposition, or debate~\cite{stone2000multiagent,wu2023autogen,du2023multiagentdebate}. These mechanisms can improve factuality or aggregate multiple perspectives, but they generally leave memory consistency implicit. ESR treats multi-agent coordination as a replicated-memory problem: the replicas need not share identical traces, but they must maintain compatible belief projections, evidence references, and safety-relevant action boundaries. Existing systems provide local memory, retrieval, traceability, or coordination; ESR combines explicit distributed memory consistency with coordinated semantic rollback.

\subsection{Conclusion}

ESR does not eliminate the need for deterministic consensus. Rather, it composes deterministic evidence replication with semantic belief replication: the evidence log remains physically ordered and linearizable, while the belief lineage is checked through verifier-bounded semantic projections. This separation lets the system reason about stochastic replicas without requiring identical prompts, token buffers, or reasoning traces.

This paper provides an initial abstraction, consistency vocabulary, and rollback mechanism for belief replication in agentic systems. Through Semantic Linearizability, Bounded Eventual Coherence, epistemic deltas, grafting, and Verifiable Semantic Rollback, ESR outlines how agent replicas can maintain compatible operational beliefs and repair faulty premises under the stated assumptions. The preliminary simulation suggests feasibility under the modeled workload, while statistically robust empirical evaluation and production deployment studies remain future work.